\documentclass{article}

\usepackage{ijcai20}
\usepackage{times}
\usepackage{soul}
\usepackage{url}
\usepackage[utf8]{inputenc}
\usepackage{microtype}
\usepackage[small]{caption}
\usepackage{amsmath}
\usepackage{amsthm}
\usepackage{amssymb}
\usepackage{mathtools}
\usepackage{booktabs}
\usepackage{nicefrac}
\usepackage{caption,subcaption}
\usepackage{placeins}
\usepackage[linesnumbered, ruled]{algorithm2e}
\usepackage[pdftex,dvipsnames]{xcolor} 
\usepackage{tcolorbox} 
\usepackage{tikz}
\usepackage{pgfplots}
\usepackage[hidelinks]{hyperref} 

\newcommand{\icol}[1]{
  \begin{array}{c}#1\end{array}%
}

\newcommand{\commentout}[1]{}
\newcommand{\abs}[1]{\ensuremath{|#1}|}

\newcommand{\true}{\ensuremath{\mathit{tt}}}
\newcommand{\false}{\ensuremath{\mathit{ff}}}
\newcommand{\modelsprop}{\models}
\newcommand{\var}{\mathit{Var}}

\newcommand{\ap}{\mathcal{P}}
\newcommand{\modelsltl}{\models}
\DeclareMathOperator{\ltluntil}{\mathsf{U}}
\DeclareMathOperator{\ltlnext}{\mathsf{X}}
\DeclareMathOperator{\ltlF}{\mathsf{F}}
\DeclareMathOperator{\ltlG}{\mathsf{G}}

\newcommand{\triggers}{\mapsto}

\newcommand{\modelspsl}{\models}

\newcommand{\matches}{\vdash}
\newcommand{\sematom}[1]{[\![#1]\!]}

\DeclareMathOperator{\remainder}{\%}

\newcommand{\flieLTL}{\texttt{LTL-Infer}}
\newcommand{\fliePSL}{\texttt{Flie-PSL}}

\newtheorem{theorem}{Theorem}
\newtheorem{proposition}{Proposition}
\newtheorem{lemma}{Lemma}
\newtheorem{observation}{Observation}{\bfseries}{\itshape}
{\bfseries}{\itshape}
\newtheorem{corollary}{Corollary}{\bfseries}{\itshape}

\usetikzlibrary{decorations.pathreplacing,angles,quotes,shapes.geometric, arrows, automata, positioning,patterns}
\pgfplotsset{compat=1.14}

\title{Learning Interpretable Models in the Property Specification Language}

\author{
Rajarshi Roy$^1$\and
Dana Fisman$^2$\And
Daniel Neider$^1$\\
\affiliations
$^1$Max Planck Institute for Software Systems, Kaiserslautern, Germany\\
$^2$Ben-Gurion University, Be'er Sheva, Israel
}

\begin{document}

\maketitle

\begin{abstract}
We address the problem of learning human-interpretable descriptions of a complex system from a finite set of positive and negative examples of its behavior.
In contrast to most of the recent work in this area, which focuses on descriptions expressed in Linear Temporal Logic (LTL), we develop a learning algorithm for formulas in the IEEE standard temporal logic PSL (Property Specification Language).
Our work is motivated by the fact that many natural properties, such as an event happening at every $n$-th point in time, cannot be expressed in LTL, whereas it is easy to express such properties in PSL.
Moreover, formulas in PSL can be more succinct and easier to interpret (due to the use of regular expressions in PSL formulas) than formulas in LTL.

Our learning algorithm builds on top of an existing algorithm for learning LTL formulas.
Roughly speaking, our algorithm reduces the learning task to a constraint satisfaction problem in propositional logic and then uses a SAT solver to search for a solution in an incremental fashion. We have implemented our algorithm and performed a comparative study between the proposed method and the existing LTL learning algorithm. Our results illustrate the effectiveness of the proposed approach to provide succinct human-interpretable descriptions from examples.
\end{abstract}

\section{Introduction}
\label{sec:intro}

Inferring an understandable and meaningful model of a complex system is an important problem in practice.
It arises naturally in various areas, including debugging, reverse engineering (e.g., of malware and viruses), specification mining for formal verification, and the modernization of legacy software.
Also, this topic clearly falls under explainable AI, as the challenge there is to obtain an \emph{explainable} model of the studied phenomena rather than a black box function implementing it.

In recent years, inferring models in Linear Temporal Logic (LTL) has crystallized as one of the most promising approaches to help humans understand the (temporal) behavior of complex systems (see the related work for a detailed discussion).
Originally developed by \cite{DBLP:conf/focs/Pnueli77} in the context of reactive systems, LTL possesses not only a host of desirable theoretical properties (e.g., the ability to effectively translate formulas into finite automata) but also features a compact, variable-free syntax and an intuitive semantics.
Specifically, these latter properties make it interesting as an interpretable description language with many applications in the area of artificial intelligence, including plan intent recognition, knowledge extraction, and reward function learning (see \cite{DBLP:conf/aips/CamachoM19} for details).

However, one of the major downsides of LTL is its limited expressive power as compared to other temporal logics. 
As a consequence, many properties that arise naturally (e.g., an event happening at every $n$-th point in time) cannot be expressed in LTL.
In fact, the class of properties that can be expressed in LTL corresponds exactly to that of star-free $\omega$-languages~\cite{DBLP:conf/focs/Wolper81}, which excludes---among others---all properties involving modulo counting.

To overcome this serious limitation, the Property Specification Language (PSL) has been proposed, which has since been adopted by IEEE as an industrial standard for expressing temporal properties~\cite{PSLieee}.
Although PSL is an extension of LTL and, hence, shares many of its beneficial properties, PSL differs from LTL in three important aspects:
\begin{enumerate}
    \item The expressive power of PSL exceeds that of LTL (it is as expressive as the full class of regular $\omega$-languages~\cite{DBLP:conf/tacas/ArmoniFFGGKLMSTVZ02}). 
    In particular, properties involving modulo counting---as mentioned above---can easily be expressed in PSL.
    \item PSL integrates easy-to-understand regular expressions in its syntax.
    \item When learning from the observed behavior of a system, models expressed in PSL can be arbitrarily more succinct than those expressed in LTL (see Proposition~\ref{prop:PSL-succinctness}).
\end{enumerate}
We believe that these three properties make PSL particularly well-suited as an interpertable description language.

\textbf{The main contribution of this paper is an algorithm for learning models (i.e., formulas) in PSL}.
Following earlier work on learning formulas expressed in LTL~\cite{DBLP:conf/fmcad/NeiderG18,DBLP:conf/aips/CamachoM19}, the precise learning problem our algorithm solves, is as follows:
given a sample $\mathcal{S}$ consisting of two finite sets of positive and negative examples, learn a PSL formula $\varphi$ that is consistent with $\mathcal{S}$ in the sense that all positive examples satisfy $\varphi$, whereas all negative examples violate $\varphi$.
Although we cannot expect algorithms that learn consistent formulas to scale as well as statistical methods that allow for misclassifications (e.g., \cite{DBLP:conf/ijcai/KimMSAS19}), being able to learn an exact model describing the given data is essential in a multitude of applications, including few-shot learning, debugging of software systems,  and many situations in which the observed data is without noise.
We refer the reader to \cite{DBLP:conf/fmcad/NeiderG18,DBLP:conf/aips/CamachoM19} for more examples where learning consistent formulas is important.

To be as general and succinct as possible, we here assume examples to be infinite, ultimately
periodic words (i.e., words of the form $uv^\omega$, where $u, v$ are a finite words and $v^\omega$ is the infinite repetition of $v$) and focus on the core fragment of PSL.
However, our algorithm can easily be adapted to learn from finite words and extends smoothly to other future-time temporal operators of PSL.

Our learning algorithm builds on top of the work by \cite{DBLP:conf/fmcad/NeiderG18} for learning formulas in LTL.
Its key idea is to reduce the learning task to a series of constraint satisfaction problems in propositional logic and use a highly-optimized SAT solver to search for a solution.
By design, our algorithm infers a smallest PSL formula that is consistent with the examples, which is a particularly valuable property in our setting: we seek to learn human-interpretable formulas and the size of the learned formula is a crucial metric for their interpretability (since larger formulas are generally harder to understand than smaller ones).
As a result from the fact that PSL makes heavy use of regular expressions, \textbf{we also obtain a learning algorithm for minimal regular expressions over finite words as a byproduct of our approach}.
Such a learning algorithm has many potential applications, for instance, in the field of natural language processing (e.g., see~\cite{DBLP:journals/computer/BartoliDLMS14}).

We empirically evaluate a prototype of our algorithm on benchmarks that reflect typical patterns of both LTL and PSL formulas used in practice.
This evaluation shows that our algorithm can infer informative PSL formulas and that these formulas are often more succinct than pure LTL formulas learned from the same examples.
Moreover, the runtime of our prototype is comparable to the state-of-the-art tool for learning LTL formulas by~\cite{DBLP:conf/fmcad/NeiderG18}.

Material and proofs that have been omitted in the paper due to space constraints can be found in the appendix.

\paragraph*{Related Work}

Learning of temporal properties has recently attracted increasing attention.
The literature in this area can be broadly structured along three dimensions.

The first dimension is the type of logic used to express models. 
Examples include learning of models expressed in Signal Temporal Logic~\cite{DBLP:journals/tac/KongJB17}, in Linear Temporal Logic~\cite{DBLP:conf/fmcad/NeiderG18,DBLP:conf/aips/CamachoM19,DBLP:conf/fdl/Riener19} and even in branching time logics, such as Computational Temporal Logic~\cite{DBLP:conf/kbse/WasylkowskiZ09}. 
To the best of our knowledge, learning of models in PSL or an equally expressive logic has not yet been considered. 

The second dimension is whether the learning algorithm requires the user to provide templates. 
Examples of algorithms that require templates are the works of \cite{DBLP:conf/memocode/LiDS11} and \cite{DBLP:conf/kbse/LemieuxPB15}, whereas the algorithms for LTL mentioned above do not require templates.
Note, however, that providing templates is often a challenging task as it requires the user to have a good understanding of the data.
By contrast, our algorithm can learn arbitrary formulas without any assistance from the user.

The third dimension distinguishes between algorithms that learn an exact model and those that learn an approximate one.
Like the majority of algorithms mentioned so far, the learning algorithm we devise in this paper is exact (i.e., it learns models that describe the data perfectly; due to our minimality constraint, however, these models generalize the data rather than overfit it).
On the other hand, there also exists work that uses statistical methods to derive approximate formulas from noisy data \cite{DBLP:conf/ijcai/KimMSAS19}.

This work is built upon the SAT-based learning algorithm by \cite{DBLP:conf/fmcad/NeiderG18}.
In fact, constraint solving is often used in learning problems.
The perhaps most prominent examples are passive automata learning~\cite{DBLP:conf/icgi/HeuleV10,DBLP:conf/atva/Neider12} and counterexample-guided inductive synthesis~\cite{DBLP:journals/cacm/AlurSFS18}.

\section{Preliminaries}
\label{se:preliminaries}

We now introduce the concepts used throughout this paper.

\paragraph*{Alphabets and Words}
An \emph{alphabet} is a finite, nonempty set $\Sigma$, whose elements are called \emph{symbols}.

A \emph{finite word} over $\Sigma$ is a finite sequence $u = a_0 \ldots a_n$ with $a_i \in \Sigma$ for $i \in \{ 0, \ldots, n \}$.
The \emph{empty word}, denoted by $\varepsilon$, is the empty sequence, and the length $|u|$ of a finite word $u$ is the number of its symbols (note that $|\varepsilon| = 0$).
Moreover, we denote the set of all words by $\Sigma^\ast$ and define $\Sigma^+ = \Sigma^\ast \setminus \{ \varepsilon \}$.

An \emph{infinite word} over $\Sigma$ is an infinite sequence $\alpha = a_0 a_1 \ldots$  with $a_i \in \Sigma$ for $i \in \mathbb N$, and $\Sigma^\omega$ denotes the set of all infinite words over $\Sigma$.
Given $u \in \Sigma^+$, the infinite word $u^\omega = uuu \ldots$ is called the \emph{infinite repetition} of $u$.
An infinite word $\alpha$ is said to be \emph{ultimately periodic} if it is of the form $\alpha = u v^\omega$ for finite words $u \in \Sigma^\ast$ and $v \in \Sigma^+$.

Given an infinite word $\alpha = a_0 a_1 \ldots \in \Sigma^\omega$ and $i, j \in \mathbb N$ with $i \leq j$, let $\alpha[i, j) = a_i \ldots a_{j-1}$ be the finite \emph{infix} of $\alpha$ from position $i$ up to (and excluding) position $j$ (note that $\alpha[i, i) = \varepsilon$).
Moreover, let $\alpha[i] = a_i$ be the symbol at position $i$ and $\alpha[i, \infty) = a_i a_{i+1} \ldots$ the infinite \emph{suffix} of $\alpha$ starting at position $i$.
We define $u[i, j)$ and $u[i]$ analogously for finite words $u \in \Sigma^\ast$ and appropriate indices $i, j$.

\paragraph*{Propositional Logic}
Let $\var$ be a set of propositional variables, which take Boolean values from $\mathbb B= \{ 0, 1 \}$.
Formulas in \emph{propositional logic}---which we denote by capital Greek letters---are inductively constructed as follows:
\[ \Phi \Coloneqq x \in \var \mid \lnot \Phi \mid \Phi \lor \Phi \]
Additionally, we add syntactic sugar and allow the formulas $\true$ (true), $\false$ (false), $\Phi_1 \land \Phi_2$, $\Phi_1 \rightarrow \Phi_2$, and $\Phi_1 \leftrightarrow \Phi_2$, which are defined as usual.

An \emph{interpretation} is a function $v \colon \var \to \mathbb B$, which assigns a Boolean value to each variable. 
The \emph{semantics} of propositional logic is given in terms of a satisfaction relation $\modelsprop$ that is inductively defined as follows:
$v \modelsprop x$ with $x \in \var$ if and only if $v(x) = 1$;
$v \modelsprop \lnot \Phi$ if and only if $v \not\modelsprop \Phi$; and
$v \modelsprop \Phi_1 \lor \Phi_2$ if and only if $v \modelsprop \Phi_1$ or $v \modelsprop \Phi_2$.
If $v \modelsprop \Phi$, we say that $v$ satisfies $\Phi$ and call it a \emph{model of $\Phi$}.
Moreover, a formula $\Phi$ is \emph{satisfiable} if there exists a model $v $ of $\Phi$.

The problem of deciding whether a propositional formula is satisfiable is the prototypical NP-complete problem~\cite{DBLP:conf/stoc/Cook71}.
Despite this fact, modern SAT solvers implement highly-optimized decision procedures that can check the satisfiability of formulas with millions of variables~\cite{DBLP:conf/aaai/BalyoHJ17}.
Moreover, virtually all SAT solvers return a model if the input-formula is satisfiable.

\paragraph*{Linear Temporal Logic}
The logic \emph{LTL}, short for \emph{Linear Temporal Logic}~\cite{DBLP:conf/focs/Pnueli77}, is an extension of propositional logic that enables reasoning about time.
The main building blocks of LTL are so-called \emph{atomic propositions}, which, intuitively, correspond to interesting properties about the system in consideration.
Given a finite set $\ap$ of atomic propositions, formulas in LTL---which we denote by small Greek letters---are inductively constructed as follows:
\[ \varphi \Coloneqq p \in \ap \mid \lnot \varphi \mid \varphi \lor \varphi \mid \ltlnext \varphi \mid \varphi \ltluntil \varphi \]
In addition to the temporal operators $\ltlnext$ (``next'') and $\ltluntil$ (``until''), we also allow the derived operators $\ltlF$ (``finally''), defined by $\ltlF \varphi \coloneqq \true \ltluntil \varphi$, and ``globally'', defined by $\ltlG \varphi \coloneqq \lnot \ltlF \lnot \varphi$ (note that our technique seamlessly extends to any future-time temporal operator, such as ``release'', ``weak until'', and so on).
Analogous to propositional logic, we also allow the formulas $\true$, $\false$, $\varphi \land \psi$, $\varphi \rightarrow \psi$, and $\varphi \leftrightarrow \psi$.

Formulas in LTL are evaluated over infinite words $\alpha \in \Sigma^\omega$ with $\Sigma = 2^\ap$ (i.e., over infinite sequences of sets of atomic propositions, modeling which propositions hold true at which points in time).
Similar to propositional logic, the semantics of LTL is defined in terms of a satisfaction relation $\modelsltl$, which formalizes when an infinite word $\alpha \in (2^\ap)^\omega$ \emph{satisfies} an LTL formula:
$\alpha \modelsltl p$ if and only if $p \in \alpha[0]$;
$\alpha \modelsltl \lnot \varphi$ if and only if $\alpha \not\modelsltl \varphi$;
$\alpha \modelsltl \varphi_1 \lor \varphi_2$ if and only if $\alpha \modelsltl \varphi_1$ or $\alpha \modelsltl \varphi_2$;
$\alpha \modelsltl \ltlnext \varphi$ if and only if $\alpha[1, \infty) \modelsltl \varphi$; and
$\alpha \modelsltl \varphi_1 \ltluntil \varphi_2$ if and only if there exists a $j \in \mathbb N$ such that $\alpha[j, \infty) \modelsltl \varphi_2$ and $\alpha[i, \infty) \modelsltl \varphi_1$ for each $i \in \{ 0, \ldots, j-1 \}$.
Note that the satisfaction of a formula, due to the temporal operators, depends on the satisfaction of its subformulas on (potentially different) infinite suffixes of $\alpha$.

It is well-known that LTL cannot express natural properties such as modulo counting. 
To alleviate this serious restriction, the Property Specification Language (PSL) has been developed (e.g., see \cite{DBLP:series/icas/EisnerF06}), which makes extensive use of regular expressions.
The remainder of this section introduces regular expressions and PSL in detail.

\paragraph*{Regular Expressions}
To simplify the definition of PSL, we define regular expressions in a slightly non-standard way.
Firstly, we use propositional formulas rather than symbols of an alphabet as atomic expressions (e.g., for $\ap = \{ p, q \}$, the formula $p \lor q$ represents the set $\{ \{ p \}, \{ q \}, \{ p, q \} \}$ of symbols from $\Sigma= 2^\ap$, whereas $p \land \lnot q$ represents the singleton set $\{ \{ p \} \}$).
Secondly, we take an operational view on regular expressions in terms of a matching relation rather than the classical view as generators of regular languages.

\emph{Regular expressions} are inductively constructed as follows, where the left-hand-side describes the construction of atomic expressions and the right-hand-side describes the construction of general regular expressions:
\begin{align*}
    \xi & \Coloneqq p \in \ap \mid \lnot \xi \mid \xi \lor \xi \hskip 1em &
    \rho & \Coloneqq \varepsilon \mid \xi \mid \rho + \rho \mid \rho \circ \rho \mid \rho^\ast
\end{align*}
As usual, the regular operator $+$ stands for choice, $\circ$ stands for concatenation, and ${}^\ast$ for finite repetition (Kleene star).
As syntactic sugar, we also allow the Boolean operators $\land$, $\rightarrow$, and $\leftrightarrow$ in atomic expressions.

Let us first give a meaning to atomic expressions.
To this end, we assign to each atomic expression $\xi$ a set $\sematom{\xi} \subseteq 2^\ap$ of symbols in the following way:
$\sematom{p} = \{ A \in 2^\ap \mid p \in A \}$; 
$\sematom{\lnot \xi} = 2^\ap \setminus \sematom{\xi}$; and
$\sematom{\xi_1 \lor \xi_2} = \sematom{\xi_1} \cup \sematom{\xi_2}$.

To define the semantics of regular expressions, we introduce a \emph{matching relation} $\matches$, which formalizes when an infix $u[i, j)$ of a finite word $u \in (2^\ap)^\ast$ \emph{matches} a regular expression.
Formally, the matching relation is defined as follows:
$u[i, j) \matches \varepsilon$ if and only if $j = i$;
$u[i, j) \matches \xi$ if and only if $j = i + 1$ and $u[i] \in \sematom{\xi}$;
$u[i, j) \matches \rho_1 + \rho_2$ if and only if $u[i, j) \matches \rho_1$ or $u[i, j) \matches \rho_2$;
$u[i, j) \matches \rho_1 \circ \rho_2$ if and only if there exists a $k \in \{ i, \ldots, j \}$ such that $u[i, k) \matches \rho_1$ and $u[k, j) \matches \rho_2$; and
$u[i, j) \matches \rho^\ast$ if and only if $j = i$ or there exists a $k \in \{ i + 1, \ldots, j \}$ such that $u[i,k) \matches \rho$ and $u[k, j) \matches \rho^\ast$.
Note that this definition applies to finite infixes $\alpha[i, j)$ of infinite words $\alpha \in (2^\ap)^\omega$ as well.

\paragraph*{Property Specification Language}
In this paper, we consider the core fragment of the \emph{Property Specification Language}~\cite{DBLP:series/icas/EisnerF06}, which we here abbreviate as \emph{PSL} for the sake of brevity.
This fragment extends LTL with a so-called \emph{triggers operator}
$\rho \triggers \varphi$
where $\rho$ is a regular expression and $\varphi$ is a PSL formula.
Intuitively, a word $\alpha \in (2^\ap)^\omega$ satisfies the PSL formula $\rho \triggers \varphi$ if $\varphi$ holds every time the regular expression $\rho$ matches on a finite prefix of $\alpha$.
To define the semantics of the triggers operator formally, we extend the satisfaction relation of LTL by
$\alpha \modelspsl \rho \triggers \varphi$ if and only if $\alpha[0, i) \matches \rho$ implies $\alpha[i-1, \infty) \modelspsl \varphi$ for all $i \in \mathbb N \setminus \{ 0 \}$.
Finally, we define the \emph{size} $|\varphi|$ of a PSL formula $\varphi$ to be the number of its unique subformulas and subexpressions.

PSL is a popular specification language in industrial applications, having been standardized by IEEE~\cite{PSLieee}.
It is as expressive as $\omega$-regular languages~\cite{DBLP:conf/tacas/ArmoniFFGGKLMSTVZ02} (i.e., languages accepted by nondeterministic Büchi automata) and, hence, exceeds the expressive power of LTL~\cite{DBLP:conf/focs/Wolper81}.
A simple property that cannot be expressed in LTL is that a proposition $p$ holds at every second point in time, which can be expressed in PSL as $(\true \circ \true)^\ast \triggers p$.

\section{The Learning Problem}
\label{sec:problem}

In this section, we formally define the learning problem studied in this paper. 
We assume the data to learn from is given as a pair $\mathcal S = (P, N)$ consisting of two finite, disjoint sets $P, N \subset \Sigma^\omega$ of ultimately periodic words such that $P \cap N \neq \emptyset$.
We call this pair a \emph{sample}.
Moreover, we say that a PSL formula $\varphi$ is \emph{consistent} with a sample $\mathcal S = (P, N)$ if $\alpha \modelspsl \varphi$ for each $\alpha \in P$ and $\alpha \not\modelspsl \varphi$ for each $\alpha \in N$.

Having defined the setting, we can now state the learning task as
\emph{``given a sample $\mathcal{S}$, compute a PSL formula of minimal size that is consistent with $\mathcal{S}$''}.
Note that this definition asks to construct a PSL formula that is minimal among all consistent formulas.
The motivation for this requirement is threefold.
Firstly, we observe that the problem becomes simple without a restriction on the size:
for $\alpha \in P$ and $\beta \in N$, one can easily construct a formula $\varphi_{\alpha, \beta}$ with $\alpha \modelsltl \varphi_{\alpha, \beta}$ and $\beta \not\modelsltl \varphi_{\alpha, \beta}$, which describes the first symbol where $\alpha$ and $\beta$ differ using a sequence of $\ltlnext$-operators and an appropriate propositional formula;
then, $\bigvee_{\alpha \in P} \bigwedge_{\beta \in N} \varphi_{\alpha, \beta}$ is trivially consistent with $\mathcal S$.
However, simply enumerating all differences of a sample is clearly of little help towards the goal of learning a descriptive model.
Secondly, small formulas are easier for humans to interpret than large ones, which justifies spending effort on learning a small (and even a smallest) formula.
Thirdly, small formulas tend to provide good generalization.

Before we explain our learning algorithm in detail, let us show that models expressed in PSL can be arbitrarily more succinct than those expressed in LTL, which follows from Theorem~4.1 of \cite{DBLP:conf/focs/Wolper81}.

\begin{proposition} \label{prop:PSL-succinctness}
Let $n\in\mathbb N$ and $\mathcal{S}_n=(P_n,N_n)$ over $\ap=\{p\}$ with $P_n= \{ \{p\}^{2n} \emptyset \{p\}^\omega \}$ and $N_n = \{ \{p\}^{2n+1}  \emptyset \{p\}^\omega \}$.
Then $(p\circ p)^\ast \triggers \ltlnext p$ is a PSL formula (of constant size) consistent with $\mathcal{S}_n$, whereas
every LTL formula that is consistent with $\mathcal{S}_n$ has size greater or equal to $2n$.  
\end{proposition}

\section{The Learning Algorithm}
\label{sec:algorithm}

The idea underlying our algorithm is to reduce the construction of a minimally consistent PSL formula to a constraint satisfaction problem in propositional logic and to use a highly-optimized SAT solver to search for a solution.
More precisely, given a sample $\mathcal S$, we construct a series $\bigl( \Phi_n^\mathcal{S} \bigr)_{n=1,2,\ldots}$ of propositional formulas that have the following properties:
\begin{enumerate}
    \item \label{itm:propositional-formula-property-1} there exists a PSL formula of size $n \in \mathbb N \setminus \{ 0 \}$ that is consistent with $\mathcal{S}$ if and only if $\Phi_n^{\mathcal{S}}$ is satisfiable; and
    \item \label{itm:propositional-formula-property-2} given a model $v$ of $\Phi_n^\mathcal{S}$, we can extract a PSL formula $\varphi_{v}$ of size $n$ that is consistent with $\mathcal S$.
\end{enumerate}

By incrementing $n$ (starting from $1$) until $\Phi_n^{\mathcal{S}}$ becomes satisfiable, we obtain an effective 
learning algorithm for models expressed in PSL, as shown in Algorithm~\ref{alg:sat-learner}.
Note that termination of this algorithm follows from the existence of a trivial solution (see Section~\ref{sec:problem}).
Moreover, its correctness follows from Properties~\ref{itm:propositional-formula-property-1} and \ref{itm:propositional-formula-property-2} of $\Phi_n^\mathcal{S}$.

\begin{algorithm}[t]
    \small
	\KwIn{A sample $\mathcal S$}
	\DontPrintSemicolon
		
	$n \gets 0$\;
	\Repeat {$\Phi_n^\mathcal{S}$ is satisfiable, say with model $v$}
	{
		$n \gets n+1$\;
		Construct $\Phi_n^\mathcal{S}$ and check its satisfiability\;
	}
	\Return $\varphi_v$\;
		
	\caption{SAT-based learning algorithm for PSL} \label{alg:sat-learner}
\end{algorithm}

\begin{theorem} \label{thm:correctness-learning-algorithm}
Given a sample $\mathcal S$, Algorithm~\ref{alg:sat-learner} terminates and outputs a minimal PSL formula that is consistent with $\mathcal S$.
\end{theorem}

\begin{corollary}
Since PSL uses regular expressions in its syntax, a simple modification of Algorithm~\ref{alg:sat-learner} learns minimal regular expressions from (finite) samples of finite words.
\end{corollary}

Roughly speaking, the formula $\Phi_n^{\mathcal S}$ is the conjunction $\Phi_n^{\mathcal S} \coloneqq \Phi^\text{str}_n \land \Phi^\text{cst}_n$, where $\Phi^\text{str}_n$ encodes the structure of the prospective PSL formula and $\Phi^\text{cst}_n$ enforces that the prospective PSL formula is consistent with the sample.
In the remainder of this section, we describe both $\Phi^\text{str}_n$ and $\Phi^\text{cst}_n$ in detail. 

\paragraph*{Structural Constraints}

The formula $\Phi^\text{str}_n$ relies on a canonical syntactic representation of PSL formulas, which we call \emph{syntax DAGs}.
A syntax DAG is essentially a syntax tree (i.e., the unique tree that is derived from the inductive definition of a PSL formula) in which common subformulas are merged.
This merging results into a directed, acyclic graph (DAG), whose number of nodes coincides with the number of subformulas of the prospective PSL formula. 
Figures~\ref{fig:syntax_tree_DAG:tree} and \ref{fig:syntax_tree_DAG:DAG} illustrate syntax trees and syntax DAGs, respectively.

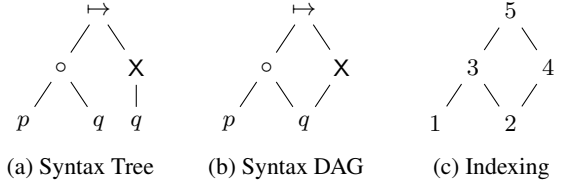
\begin{figure}[t]
	\centering
    \begin{subfigure}[b]{0.25\linewidth}
	    \centering
	    \begin{tikzpicture}[-,auto, scale=1]
    		\node at (0,0)         (A)  {\small{$\triggers$}};
	    	\node at (-0.5,-0.75)  (B)  {\small{$\circ$}};
			\node at (0.5,-0.75)   (C)  {\small{$\ltlnext$}};
    		\node at (-1,-1.5)     (D)  {\small{$p$}};
	    	\node at (0,-1.5)      (E)  {\small{$q$}};
		    \node at (0.5,-1.5)    (F)  {\small{$q$}};
    		\path 	(A) edge (B)
	    				edge (C)
		    		(B) edge (D)
    					edge (E)
	    			(C) edge (F);
	    \end{tikzpicture}
	    \caption{Syntax Tree}
	    \label{fig:syntax_tree_DAG:tree}
    \end{subfigure}
    \hskip 5mm
    %
    \begin{subfigure}[b]{0.25\linewidth}
	    \centering
        \begin{tikzpicture}[-,auto, scale=1]
    		\node at (0,0)         (A)  {\small{$\triggers$}};
	    	\node at (-0.5,-0.75)  (B)  {\small{$\circ$}};
		    \node at (0.5,-0.75)   (C)  {\small{$\ltlnext$}};
    		\node at (-1,-1.5)     (D)  {\small{$p$}};
	    	\node at (0,-1.5)      (E)  {\small{$q$}};
    		\path 	(A) edge (B)
	    				edge (C)
		    		(B) edge (D)
    					edge (E)
	    			(C) edge (E);
	    \end{tikzpicture}
	    \caption{Syntax DAG}
	    \label{fig:syntax_tree_DAG:DAG}
    \end{subfigure}
    \hskip 5mm
    %
    \begin{subfigure}[b]{0.25\linewidth}
	    \centering
	    \begin{tikzpicture}[-,auto, scale=1]
    		\node at (0,0)         (A)  {\small{$5$}};
	    	\node at (-0.5,-0.75)  (B)  {\small{$3$}};
		    \node at (0.5,-0.75)   (C)  {\small{$4$}};
    		\node at (-1,-1.5)     (D)  {\small{$1$}};
	    	\node at (0,-1.5)      (E)  {\small{$2$}};
    		\path 	(A) edge (B)
	    				edge (C)
		    		(B) edge (D)
    					edge (E)
	    			(C) edge (E);
	    \end{tikzpicture}
	    \caption{Indexing} 
	    \label{fig:syntax_tree_DAG:labeling}
	\end{subfigure}
	\caption{Different representations of the PSL formula $(p\circ q) \triggers \ltlnext q$}
	\label{fig:syntax_tree_DAG}
	
\end{figure}

To simplify our encoding, we assign a unique identifier $k \in \{ 1, \ldots, n \}$ to each node of a syntax DAG such that (a) the identifier of the root is $n$ and (b) the identifier of an inner node is larger than the identifiers of its children (see Figure~\ref{fig:syntax_tree_DAG:labeling}).
Note that this encoding entails that Node~$1$ is always a leaf, which is necessarily labeled with an atomic proposition.

Let now $\Lambda_R = \{\lnot, \lor, +, \circ, {}^\ast\} \cup \ap$ be the set of operators and atomic propositions that can appear in regular expressions and $\Lambda_P = \Lambda_R \cup \{ \ltlnext, \ltluntil, \triggers \}$ be the set of all PSL operators and atomic propositions.
Then, we can encode a syntax DAG using the following propositional variables:
\begin{itemize}
    \item $x_{k, \lambda}$ where $k\in \{1, \ldots, n\}$ and $\lambda \in \Lambda_P$
    \item $l_{k, \ell}$ where $k\in \{2, \ldots,n\}$ and $\ell \in \{1, \ldots, k-1 \}$
    \item $r_{k, \ell}$ where $k\in \{2, \ldots,n\}$ and $\ell \in \{1, \ldots, k-1 \}$
\end{itemize}
Intuitively, the variables $x_{k, \lambda}$ encode the labeling of a syntax DAG in the sense that if $x_{k, \lambda}$ is set to true, then node $k$ is labeled by $\lambda$.
Similarly, the variables $l_{k, \ell}$ and $r_{k, \ell}$ encode the left and right child of node $k$, respectively.
By convention, we ignore the variables $r_{k, \ell}$ (resp.\ $r_{k, \ell}$ and $l_{k, \ell}$) if node $k$ is labeled with an unary operator (resp.\ an atomic proposition).

To enforce that these variables in fact encode a syntax DAG, we first need to make sure that for each $k \in \{ 1, \ldots, n \}$ there exists precisely one $\lambda \in \Lambda_P$ such that $x_{k, \lambda}$ is set to true.
This can be done with the following constraint:
\begin{align*}
    \Big[ \bigwedge_{1\leq k \leq n} \bigvee_{\lambda \in \Lambda_P} x_{k, \lambda} \Big] \land \Big[ \bigwedge_{1 \leq k \leq n} \bigwedge_{\lambda \neq \lambda' \in \Lambda_P} \lnot x_{k, \lambda} \lor \lnot x_{k, \lambda'} \Big]   
\end{align*}
Similarly, we assert that for each $k \in \{ 1, \ldots, n \}$ there exists precisely one $\ell \in \{1, \ldots, k-1 \}$ and one $\ell' \in \{1, \ldots, k-1 \}$ such that $l_{k, \ell}$ and $r_{k, \ell'}$ is set to true, respectively.

Next, we have to ensure that the labeling of the syntax DAG respects the type of the operators (e.g., children of a regular expression are also regular expressions).
The constraint below exemplifies this for the concatenation operator $\circ$:
\begin{align*}
    \bigwedge_{\substack{1 \leq k \leq n \\ 1 \leq \ell, \ell^\prime< k}} [ x_{k, \circ} \land l_{k, \ell} \land r_{k,\ell'} ] \rightarrow \Big[ \bigvee_{\lambda \in \Lambda_R} x_{\ell, \lambda} \land \bigvee_{\lambda\in\Lambda_R} x_{\ell^{\prime}, \lambda} \Big] 
\end{align*}
We add analogous constraints for all other operators.
Note that the constraint for the triggers operator is slightly different as it combines a regular expression and a PSL formula.

It is left to enforce that Node~$1$ is always labeled with an atomic proposition.
We do so using the constraint $\bigvee_{p \in \ap} x_{1, p}$.

Finally, let $\Phi^\text{str}_n$ be the conjunction of all constraints discussed above.
Then, one can construct a syntax DAG from a model $v$ of  $\Phi^\text{str}_n$ in a straightforward manner: label Node~$k$ with the unique $\lambda \in \Lambda_P$ such that $v(x_{k, \lambda}) = 1$, designate Node~$n$ as the root, and arrange the nodes as described uniquely by $v(l_{k, \ell})$ and $v(r_{k, \ell})$.
Subsequently, we can derive a PSL formula from this syntax DAG, which we denote by $\varphi_v$.
To ensure that $\varphi_v$ is consistent with $\mathcal S$, we add further constraints (i.e., a formula $\Phi^\text{cst}_n$), which we describe next.

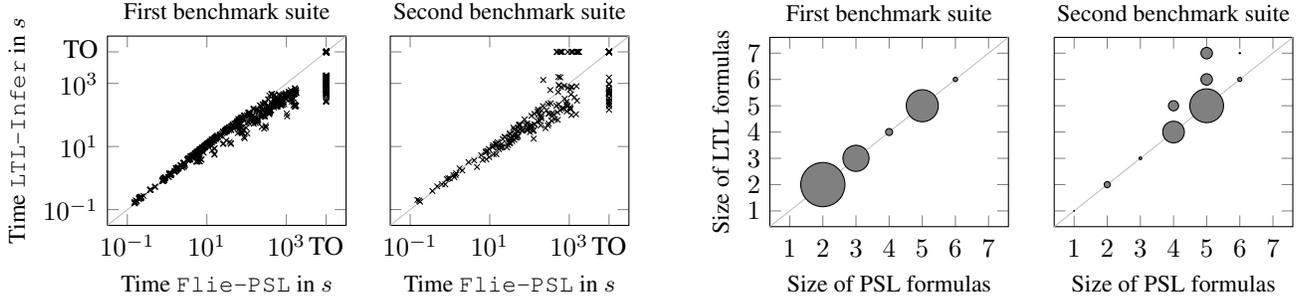
\begin{figure*}[t]
    \centering
    %
        \begin{subfigure}[b]{0.27\textwidth}
        \centering
	       \begin{tikzpicture}
                \begin{loglogaxis}[
                        height=41mm,
                        xmin=1e-1, ymin=1e-1,
                        enlarge x limits=true, enlarge y limits=true,
                        xlabel = {Time \fliePSL\ in $s$},
                        ylabel = {Time \flieLTL\ in $s$},
                        xtickten={-1, 1, 3},%
                        ytickten={-1, 1, 3},%
                        extra x ticks={10000}, extra x tick labels={\strut TO},%
                        extra y ticks={10000}, extra y tick labels={\strut TO},%
                        label style={font=\footnotesize},
                        x tick label style={font=\strut},
                        x label style = {yshift=1mm},
                        title = {First benchmark suite},
                        title style = {font=\footnotesize, inner sep=0pt}
                ]
                \addplot[
                        scatter=false,
                        only marks,
                        mark=x,
                        mark size=1.5,
                ]
                        table [col sep=comma,x={PSL-Time taken(2loops+unfoldedTraces)},y={LTL-Time taken}] {PSLvsLTLstats_LTL.csv};

                \draw[black!25] (rel axis cs:0, 0) -- (rel axis cs:1, 1);

                \end{loglogaxis}
            \end{tikzpicture}    
        \end{subfigure}  
    %
        \begin{subfigure}[b]{0.21\textwidth}
        \centering
            \begin{tikzpicture}
                \begin{loglogaxis}[
                        height=41mm,
                        xmin=1e-1, ymin=1e-1,
                        enlarge x limits=true, enlarge y limits=true,
                        xlabel = {Time \fliePSL\ in $s$},
                        xtickten={-1, 1, 3},%
                        ytickten={-1, 1, 3},%
                        extra x ticks={10000}, extra x tick labels={\strut TO},
                        extra y ticks={10000}, extra y tick labels=\empty,
                        xlabel near ticks,
                        label style={font=\footnotesize},
                        x tick label style={font=\strut},
                        yticklabels=\empty,
                        x label style = {yshift=1mm},
                        title = {Second benchmark suite},
                        title style = {font=\footnotesize, inner sep=0pt}
                ]

                \addplot[
                        scatter=false,
                        only marks,
                        mark=x,
                        mark size=1.5,
                ]
                        table [col sep=comma,x={PSL-Time taken(2loops+unfoldedTraces)},y={LTL-Time taken}] {PSLvsLTLstats_PSL.csv};

                \draw[black!25] (rel axis cs:0, 0) -- (rel axis cs:1, 1);

                \end{loglogaxis}
            \end{tikzpicture}    
        \end{subfigure}
    \hskip 6mm
    %
        \begin{subfigure}[b]{0.24\textwidth}
        \centering
        \tikzset{every mark/.append style={fill=red}}
	        \begin{tikzpicture}
                \begin{axis}[
                        height=41mm,
                        xmin=1, xmax=7,
                        ymin=1, ymax=7,
                        enlarge x limits=true, enlarge y limits=true,
                        xlabel = {Size of PSL formulas},
                        ylabel = {Size of LTL formulas},
                        xtick={1,2,3,4,5,6,7},
                        ytick={1,2,3,4,5,6,7},
                        label style={font=\footnotesize},
                        x tick label style={font=\strut},
                        x label style = {yshift=1mm},
                        title = {First benchmark suite},
                        title style = {font=\footnotesize, inner sep=0pt}
                ]
                \addplot[scatter=true,
                        only marks,
                        mark options={fill=gray},
                        visualization depends on = {35*\thisrow{Frequency} \as \perpointmarksize},
                        scatter/@pre marker code/.style={/tikz/mark size=\perpointmarksize},
                        scatter/@post marker code/.style={}] table [col sep=comma,x={PSL-Size},y={LTL-Size},meta index=2] {SizePairFrequency_LTL.csv};
                        every mark/.append style={solid, fill=gray},
                \draw[black!25] (rel axis cs:0, 0) -- (rel axis cs:1, 1);
                \end{axis}
            \end{tikzpicture}
        \end{subfigure}
    \hskip 1mm
        \begin{subfigure}[b]{0.20\textwidth}
        \centering
            \begin{tikzpicture}
                \begin{axis}[
                        height=41mm,
                        xmin=1, xmax=7,
                        ymin=1, ymax=7,
                        enlarge x limits=true, enlarge y limits=true,
                        xlabel = {Size of PSL formulas},
                        xtick={1,2,3,4,5,6,7},%
                        ytick={1,2,3,4,5,6,7},%
                        yticklabels=\empty,
                        label style={font=\footnotesize},
                        x tick label style={font=\strut},
                        x label style = {yshift=1mm},
                        title = {Second benchmark suite},
                        title style = {font=\footnotesize, inner sep=0pt}
                ]
                \addplot[scatter=true,
                        only marks,
                        mark=*,
                        mark options={fill=gray},
                        point meta=explicit,
                        visualization depends on = {50*\thisrow{Frequency} \as \perpointmarksize},
                        scatter/@pre marker code/.style={/tikz/mark size=\perpointmarksize},
                        scatter/@post marker code/.style={}] table [col sep=comma,x={PSL-Size},y={LTL-Size},meta index=2] {SizePairFrequency_PSL.csv};
                \draw[black!25] (rel axis cs:0, 0) -- (rel axis cs:1, 1);
                \end{axis}
            \end{tikzpicture}    
        \end{subfigure} 

    \caption{Comparison of \fliePSL\ and \flieLTL. The size of the bubbles reflects the number of formulas. ``TO'' indicates timeouts.}
    \label{fig:experiments}
\end{figure*}

\paragraph*{Constraints for Consistency}

To construct the propositional formula $\Phi^\text{cst}_n$, we exploit a simple observation about PSL.

\begin{observation} \label{obs:repeat}
Let $uv^\omega \in (2^\ap)^\omega$ and $\varphi$ be a PSL formula.
Then, $uv^\omega[|u|+i,\infty) = uv^\omega[|u|+j,\infty)$ for $j\equiv i\mod{\abs{v}}$.
Thus, $uv^\omega[|u|+i,\infty) \modelspsl \varphi$ if and only if $uv^\omega[|u|+j,\infty) \modelspsl \varphi$.
\end{observation}

Intuitively, Observation~\ref{obs:repeat} states that there exists only a finite number of distinct infinite suffixes of a word $uv^\omega$, which eventually repeat periodically.
Since the semantics of PSL is defined in terms of the suffixes of a word, we can in fact determine whether an infinite word $uv^\omega$ satisfies a PSL formula based only on its finite prefix $uv$.
To illustrate this claim, consider the formula $\ltlnext \varphi$ and suppose that we want to determine whether $uv^\omega[|uv|-1, \infty) \modelspsl \ltlnext \varphi$ holds (i.e., satisfaction of $\ltlnext \varphi$ is checked at the end of the prefix $uv$).
Then, Observation~\ref{obs:repeat} allows us to reduce this question to checking whether $uv^\omega[|u|, \infty) \modelspsl \varphi$ holds, instead of the original semantics of the $\ltlnext$-operator, which depends on whether $uv^\omega[|uv|, \infty) \modelspsl \varphi$ is satisfied.

For reasoning about matchings of regular expressions, however, it is not enough to just consider the prefix $uv$.
For instance, consider the ultimately periodic word $uv^\omega = \emptyset \{p\} (\emptyset)^\omega$ and the PSL formula $\varphi \coloneqq (\true\circ\true)^\ast \triggers p$ (stating that $p$ is true at every second position).
By just considering the prefix $uv = \emptyset \{ p \} \emptyset$, it seems that $uv^\omega \modelspsl \varphi$.
However, \emph{unrolling} the repeating part $v = \emptyset$ once more, resulting in the prefix $uvv = \emptyset \{p\} \emptyset \emptyset$, immediately shows that $uv^\omega \not\modelspsl \varphi$.

Similar to Observation~\ref{obs:repeat}, the next lemma provides a bound $b \in \mathbb N$ on the number of unrollings required to gather enough information to determine the satisfaction of a triggers operator.
This bound depends on the number $n$ of nodes of the syntax DAG and the function $M_{u,v} \colon\mathbb N \to \mathbb N$ defined by
\begin{align*}
    M_{u,v}(j)= 
    \begin{cases}
	    j & \text{if $j < \abs{uv}$; and } \\
	    \abs{u}+((j-\abs{u}) \remainder \abs{v}) & \text{if $j \geq \abs{uv}$,}
    \end{cases}
\end{align*}
where $a \remainder b$ is the remainder of the division $\nicefrac a b$.
Intuitively, $M_{u,v}$ maps a position $j$ in the word $uv^\omega$ to an appropriate position within the prefix $uv$.
The lemma uses finite automata as representations of regular expressions to derive the bound.

\begin{lemma} \label{lem:unrolling_words}
Let $uv^\omega \in (2^\ap)^\omega$, $\psi = \rho \triggers \varphi$ with $|\psi| = n$, and $b = 2^{n}+1$.
Then, $uv^\omega[i,\infty) \modelspsl \psi$ if and only if for all $j \leq \abs{u} + b \abs{v}$, $uv^\omega[i,j) \matches \rho$ implies $uv^\omega[M_{u,v}(j-1), \infty) \modelspsl \varphi$.
\end{lemma}

Note an important property of Lemma~\ref{lem:unrolling_words}: reasoning about regular expressions and the triggers operator $\triggers$ requires us to consider the prefix $uv^b$, while the prefix $uv$ is sufficient for reasoning about the remaining PSL operators. 

Towards the definition of the formula $\Phi^\text{cst}_n$, we construct for each ultimately periodic word $uv^\omega$ in $\mathcal S$ a propositional formula $\Phi_n^{u,v}$ that tracks the satisfaction of the PSL
formula encoded by $\Phi^\text{cst}_n$ (and all its subformulas\slash subexpressions) on $uv^\omega$.
Each of these formulas is built over auxiliary variables:
\begin{itemize}
	\item $y^{u,v}_{i,k}$ with $0 \leq i< \abs{uv}$ and $k \in \{ 1, \ldots, n \}$
	\item $z^{u,v}_{i,j,k}$ with $0\leq i\leq j\leq \abs{uv^{b}}$, $b=2^{n}+1$ as in Lemma~\ref{lem:unrolling_words}, and $k \in \{ 1, \ldots, n \}$
\end{itemize}
The meaning of these variables is that $y^{u,v}_{i,k}$ is set to true if and only if $uv^\omega[i, \infty)$ satisfies the PSL formula rooted at Node~$k$ (if that node is labeled with a PSL operator);
similarly, $z^{u,v}_{i,j,k}$ is set to true if and only if $uv^\omega[i, j)$ matches the regular expression rooted at Node~$k$ (if that node is labeled with a regular expression operator).
Note that we have to create both the variables $y^{u,v}_{i,k}$ and $z^{u,v}_{i,j,k}$ for each node since the ``type'' of a node is determined dynamically during SAT solving.

It is left to enforce that the variables $y^{u,v}_{i,k}$ and $z^{u,v}_{i,j,k}$ have the desired meaning.
For the Boolean and temporal operators (except the triggers operator), we reuse the constraints proposed by~\cite{DBLP:conf/fmcad/NeiderG18}.
For instance, the constraint for the atomic propositions is 
\begin{align*}
    \bigwedge_{1 \leq k \leq n} \bigwedge_{p \in \ap} x_{k, p} \rightarrow \bigwedge_{0 \leq i< \abs{uv}}
    \begin{cases}
        y^{u,v}_{i,k} & \text{if $p\in uv[i]$; and} \\
        \lnot y^{u,v}_{i,k} & \text{if $p \notin uv[i]$.}
    \end{cases}
\end{align*}
Intuitively, this constraint states that if Node~$k$ is labeled with the atomic proposition $p \in \ap$, then the variables $y^{u,v}_{i,k}$ capture precisely the presence or absence of $p$ in the $k$-th position of the prefix $uv$.
Similarly, the constraint for the $\ltlnext$-operator is
\begin{multline*}
	\bigwedge_{\substack{1\leq k \leq n,~ 1\leq \ell< k}} [x_{k,\ltlnext} \wedge l_{k,\ell}]
	\rightarrow\\ 
	\Big[ \bigwedge_{0\leq i< \abs{uv}-1} \Big[ y^{u,v}_{i,k} \leftrightarrow y^{u,v}_{i+1,\ell} \Big] \Big] \land \Big[ y^{u,v}_{\abs{uv}-1,k}\leftrightarrow y^{u,v}_{\abs{u},\ell} \Big],
\end{multline*}
which states that if Node~$k$ is labeled with $\ltlnext$ and its left child is Node~$\ell$, then the satisfaction of the formula rooted at Node~$k$ at time $i$ (i.e., $y^{u,v}_{i,k}$) equals the satisfaction of the subformula rooted at Node~$\ell$ at time $i+1$ (i.e., $y^{u,v}_{i+1,\ell}$), except at time $|uv|-1$, where it ``wraps around'' to time $|u|$ (see Observation~\ref{obs:repeat}).

The constraints for regular expressions follow the definition of the matching relation $\matches$ and refer to the variables $z^{u,v}_{i,j,k}$ rather than $y^{u,v}_{i,k}$.
Exemplarily, we here present the constraints for the concatenation operator $\circ$:
\begin{multline*}
	\bigwedge_{\substack{1\leq k \leq n,~ 1\leq \ell, \ell^{\prime}<k}} [ x_{k,\circ} \land l_{k,\ell}\wedge r_{k,\ell^{\prime}} ] \rightarrow \\
	\bigwedge_{\substack{0\leq i\leq j \leq \abs{uv^b} }}\Big[z^{u, v}_{i,j,k} \leftrightarrow \bigvee_{i\leq t\leq j}z^{u, v}_{i,t,\ell}\land z^{u, v}_{t,j,\ell^{\prime}} \Big]
\end{multline*}
Constraints for the other regular operators are analogous.

Finally, the constraint below captures the semantics of the triggers operator $\triggers$
by relating the variables $y^{u,v}_{i,k}$ and $z^{u,v}_{i,j,k}$.
\begin{multline*}
    \bigwedge_{\substack{1\leq k \leq n,~ 1\leq \ell, \ell^{\prime}< k}} [ x_{k,\triggers} \land l_{k,\ell}\wedge r_{k,\ell^{\prime}} ] \rightarrow\\
    \bigwedge\limits_{0\leq i< \abs{uv}}\Big[y^{u,v}_{i,k}\leftrightarrow \bigwedge\limits_{i\leq j\leq \abs{uv^{b}}}\Big[z^{u,v}_{i,j,\ell}\rightarrow y^{u,v}_{M_{u, v}(j-1),\ell^{\prime}} \Big] \Big]
\end{multline*}

As the final step, we define the formula $\Phi^\text{cst}_n$ by
\begin{align*}
    \Phi^\text{cst}_n \coloneqq \Big[\bigwedge\limits_{uv^\omega\in P} \Phi_n^{u, v} \wedge y^{u, v}_{0,n} \Big] \land \Big[ \bigwedge_{uv^\omega\in N} \Phi_n^{u, v} \land \lnot y^{u, v}_{0,n} \Big],
\end{align*}
which enforces that all positive words in $\mathcal S$ satisfy the prospective PSL formula ($y^{u, v}_{0,n}$ has to be true), while all negative words violate it ($y^{u, v}_{0,n}$ has to be false).

\section{Evaluation}
\label{sec:evaluation}

We have implemented a prototype of our learning algorithm, named \fliePSL\ (Formal Language Inference Engine for PSL), which we will make publicly available.
This prototype is written in Python and uses Z3~\cite{DBLP:conf/tacas/MouraB08} as SAT solver.
Deviating slightly from the general algorithm presented in Section~\ref{sec:algorithm}, we have implemented the following improvement:
instead of generating the variables $y^{u,v}_{i,k}$ and $z^{u,v}_{i,j,k}$ for each node, we generate the latter variables (and their constraints) only for $0 \leq m < n$ nodes and the former variables (and their constraints) for the remaining $n-m$ nodes.
This effectively limits the size of a regular expression in the final PSL formula to $m$.
To obtain a complete algorithm, we iterate over all valid values for $m$ before increasing $n$.

To assess the performance of our prototype, we have compared it to an implementation of the LTL learning algorithm by \cite{DBLP:conf/fmcad/NeiderG18}, which we call \flieLTL\ for brevity.
To make this comparison as fair as possible, we have used two benchmark suites.
The first benchmark suite is taken directly from \cite{DBLP:conf/fmcad/NeiderG18} and contains $1217$ samples, which were generated from common LTL properties.
The second benchmark suite is meant to simulate real-world PSL use-cases and contains $390$ synthetic samples, which we have generated from PSL formulas that commonly appear in practice (e.g., $(p_1 \circ p_2)^\ast \triggers q$; see \cite{DBLP:series/icas/EisnerF06} for more examples).
Our procedure to generate these samples is similar to the one by \cite{DBLP:conf/fmcad/NeiderG18} and proceeds as follows:
firstly, we select a formula $\varphi$ from our pool of PSL formulas;
secondly, we generate up to $500$ ultimately periodic words $uv^\omega$ with $\abs{u}+\abs{v}\leq 15$;
thirdly, we partition these words into sets $P$ and $N$ depending on their satisfaction of $\varphi$. 
In total, the median size of the samples in the second benchmark suite is $100$ words.
All experiments were conducted on a single core of an Intel Xeon E7-8857~V2 CPU (at $3.6$\,GHz) with a timeout of $1800\,s$.

The two diagrams on the left-hand-side of Figure~\ref{fig:experiments} compare the runtime of \fliePSL\ and \flieLTL\ on the first and second benchmark suite, respectively.
In general, \fliePSL\ is moderately slower than \flieLTL\ and timed out $1.34$ times more often (\fliePSL\ timed out $38.4\%$ and $56.2\%$ of the times on the first and second benchmark suite, respectively, whereas \flieLTL\ timed out $24.8\%$ and $53.6\%$ of the times).
This came as a surprise to us because the SAT encoding in the case of PSL is much more involved than the one for LTL.
In fact, there were even $25$ benchmarks on which \fliePSL\ outperformed \flieLTL\ because it was able to learn smaller formulas. 

The two diagrams on the right-hand-side of Figure~\ref{fig:experiments} compare the size of the formulas learned by both tools.
On the first benchmark suite, we observe that \fliePSL\ mainly produced pure LTL formulas of the same size as \flieLTL\ (a likely explanation for this is that these benchmarks have explicitly been designed to capture LTL properties).
However, on $68$ benchmarks of the second suite, \fliePSL\ learned PSL formulas that use non-LTL operators and was able to recover the exact PSL property that was used to generate the sample in $40$ of the benchmarks.
Overall, \fliePSL\ learned a smaller formula than \flieLTL\ for $52$ benchmarks.

\section{Conclusion}
\label{sec:conclusion}

We have developed an algorithm for learning human-interpretable models expressed in PSL and have shown empirically that this algorithm infers interesting PSL formulas with only little overhead as compared to learning LTL formulas.

An interesting direction for future work would be to syntactically restrict the class of regular expressions so as to reduce the number $b$ of unrolling required for the variables $z^{u,v}_{i,j,k}$ and, hence, improve performance.
Moreover, we plan to extend our algorithm to be able to handle noisy data and, orthogonally, to learn models expressed as $\omega$-regular expressions.

\bibliographystyle{named}
\bibliography{bib}

\appendix
\clearpage

\section{Proofs regarding unrolling of words}\label{proof:unrolling_words}

In this section, we prove Lemma~\ref{lem:unrolling_words} used in Section~\ref{sec:algorithm}, which provides a bound for unrolling of ultimately periodic words to check consistency with triggers operator. 
We know that triggers operator uses regular expression and arguing about matching with a regular expression becomes easier when the regular expression is viewed as a finite state acceptor. 
As a result, in the proofs here, instead of  deriving the bound $b$ in terms of the size of the syntax DAG for $\rho\triggers\varphi$, we use the size $m$ of the minimal DFA (size refers to the number of states of the DFA) for $\rho$ to find the appropriate bound.
Nonetheless, the bound has to be related to the size of the syntax DAG and thus, we use a loose upper bound of $2^n$ for $m$ in Lemma~\ref{lem:unrolling_words}, since, size of minimal DFA can be exponentially larger than its regular expression. 
Tighter upper bounds for $m$ can be found in \cite{DBLP:journals/corr/GruberH14}.

\begin{lemma}\label{lem:bounding_regex_matches}
Let $uv^{\omega}[i,j)\matches \rho$, for some $i$, $j$ where, $i\leq\abs{u}< j$. Then, there exists $k\in\mathbb{N}$, $\abs{u}\leq k\leq\abs{u}+m\abs{v}$ such that $uv^\omega[i,k)\matches \rho$ and $k\equiv j \mod \abs{v}$.
\end{lemma}

\begin{proof}
If $\abs{u}\leq j\leq \abs{u}+m\abs{v}$, we are done since we simply take $k=j$. 
However, if $j>\abs{u}+m\abs{v}$, finding the suitable $k$ is slightly more involved. 

The first observation we make is that, since $uv^{\omega}[i,j)\matches \rho$, there is an accepting run of the DFA $\mathcal{A}$ (of size $m$) for $\rho$ on $uv^{\omega}[i,j)$. 
Fig \ref{fig:lemma1} provides a pictorial depiction of the run. Notice that the portion of the run on $v^{\omega}[0,j)$ itself, has a length greater than $m\abs{v}$. 
We consider this portion of the run to be a sequence of tuples of the form $(\mathit{state}, \mathit{index})$, where $index$ refers to the position in $v$ which will be read next by the $\mathit{state}$ of the automaton. 
Now, due to pigeonhole principle, if this run is longer than $m\abs{v}$, then there exists a tuple which repeats during the run. Let $(q,l)$ be the tuple which repeats and let the run from the first occurrence of $(q,l)$ to the second, be referred to as $R$. 
Notice that due to the deterministic nature of the automaton, $R$ repeats during the rest of the run. Hence, if a final state $q_f$ occurs after $m\abs{v}$ steps, there must be a tuple $(q_f,l_f)$ which belongs to the run $R$. 
Clearly, $(q_f,l_f)$ must have been also visited during the first occurrence of $R$, which happens within the first $m\abs{v}$ steps of the entire run on $v^\omega$. 
Thus, we get a prefix $uv^\omega[i,k)\matches \rho$, where $k\leq \abs{u}+m\abs{v}$.
Moreover, $uv^\omega[i,j)$ and $uv^\omega[i,k)$ terminate at the same position in $v$, meaning $k\equiv j \mod \abs{v}$.
\end{proof}

\begin{figure}
    \centering
\begin{tikzpicture}[scale=1.25]

\draw [->, thick] (-3.8,0) -- (3,0);
\draw [<->,>=stealth] (-3.5,+0.4) -- node[above]{\small{length of run:} $m\abs{v}$} (0.5,0.4);

\foreach \i in {-3.65, -0.3, 1.2}
    \draw [thick] (\i,0.05) -- (\i,-0.05);
\node [below] at (-3.65,-0.05) {\tiny{$i$}};
\node [below] at (-0.3,-0.05) {\tiny{$k$}};
\node [below] at (1.2,-0.05) {\tiny{$j$}};

\node [above] at (-3.65,+0.05) {\small{$u$}};
\foreach \i in {-3.25,-2.75,-2.25,-1.75,-1.25,-0.75,...,2.25}
    \node [above] at (\i,+0.05) {\small{$v$}};
\node [above] at (2.75,+0.05) {\small{$\cdots$}};

\fill[black] (-3.8,0) circle (0.4 mm);
\foreach \i in {-3.5,-3.0,...,2.5}
    \fill[black] (\i,0) circle (0.4 mm);
    
\draw [->] (-3.8,-0.4) -- (3,-0.4);
\foreach \i in {-3.5,-3.4,-3.3,...,2.5}
    \draw (\i,-0.35) -- (\i,-0.45);
\foreach \i in {-3.5,-1.2,0.3,1.8,-0.3,1.2}
    \draw [thick] (\i,-0.32) -- (\i,-0.48);

\node [below, fill=Orchid, rounded corners=0.3cm, scale=0.8] at (-3.5,-0.5) {$\icol{q_0\\0}$};
\node [below, fill=ProcessBlue, rounded corners=0.3cm, scale=0.8] at (-1.2,-0.5) {$\icol{q\\l}$};
\node [below, fill=YellowGreen, rounded corners=0.3cm, scale=0.8] at (-0.3,-0.5) {$\icol{q_f\\l_f}$};
\node [below, fill=ProcessBlue, rounded corners=0.3cm, scale=0.8] at (0.3,-0.5) {$\icol{q\\l}$};
\node [below, fill=YellowGreen, rounded corners=0.3cm, scale=0.8] at (1.2,-0.5) {$\icol{q_f\\l_f}$};
\node [below, fill=ProcessBlue, rounded corners=0.3cm, scale=0.8] at (1.8,-0.5) {$\icol{q\\l}$};

\draw[decoration={brace,amplitude=6pt, raise=5pt, mirror},decorate]
  (-1.2,-1.1) -- node[below=10pt] {$R$} (0.3,-1.1);
\draw[decoration={brace,amplitude=6pt, raise=5pt, mirror},decorate]
  (0.3,-1.1) -- node[below=10pt] {$R$} (1.8,-1.1);

\end{tikzpicture}

    \caption{The run of the DFA for $\rho$ on $uv^\omega[i,j)$, where $R$ is the repeating run staring at $(q,l)$. The first occurrence of $R$ happens within (and including) first $m\abs{v}$ steps of the run on $v^\omega$ and hence, the first occurrence of state $q_f$ also happens within that portion as well.}
    \label{fig:lemma1}
\end{figure}
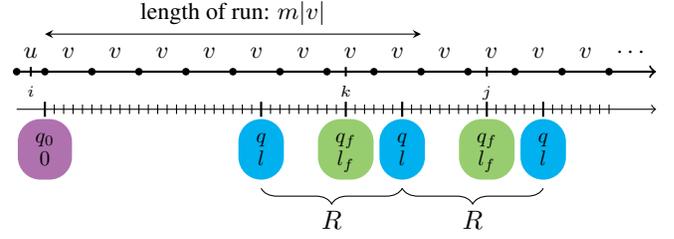

\begin{lemma}\label{lem:unrolling_words_general}
Let $b = m+1$. Then, we have $uv^\omega[i,\infty)\modelspsl  \rho\triggers \varphi$, where $0
\leq i\leq \abs{uv}-1$,  if and only if for all $j<\abs{u}+b\abs{v},\ uv^\omega[i,j)\matches \rho$  implies $uv^\omega[j-1,\infty) \modelspsl \varphi$.
\end{lemma}

\begin{proof}
The forward direction of the theorem follows from the semantics of triggers operator.

The other direction is a direct consequence of Lemma~\ref{lem:bounding_regex_matches}. The additional $\abs{v}$ term in the bound that appears in the theorem because of the fact that here, $i$ could range between $0$ and $\abs{uv}-1$, unlike in Lemma~{\ref{lem:bounding_regex_matches}}. When $i>\abs{u}$, similar argument as in the lemma works just by considering $u=v$.
\end{proof}

As evident, Lemma~\ref{lem:unrolling_words_general} provides an upper bound on the number of unrollings of $uv^\omega$ required to check consistency for triggers operator, in terms of $m$. This result holds for any upper bound of $m$, as discussed at the beginning of the section. Therefore, we derive Lemma~\ref{lem:unrolling_words} from
Lemma~\ref{lem:unrolling_words_general} (along with Observation~\ref{obs:repeat} to construct the function $M_{u,v}$), with a suitable upper bound of $m\leq 2^n$ in terms of size of the syntax DAG of the formula.

\section{List of all constraints used in the SAT encoding}
In this section, we have listed down all the constraints (in Figure~\ref{fig:all-constraints}) that have been used to construct $\Phi^\mathcal{S}_n$ appearing in Algorithm~\ref{alg:sat-learner}. We have partitioned the constraints into three different tables depending on the type of the constraint.

The first table provides the structural constraints used for encoding of the syntax DAG. 
In particular, Formulas~\ref{eq:label_uniqueness},~\ref{eq:left_child_uniqueness}, and~\ref{eq:right_child_uniqueness} ensure uniqueness of the label, left child, and right child of a node respectively. On the other hand, Formulas~\ref{eq:regex_ordering},~\ref{eq:LTL_ordering}, and~\ref{eq:triggers_ordering} are the ordering constraints for regular operators, LTL operators, and triggers operator respectively. 
Finally, Formula~\ref{eq:prop_ordering} asserts that the Node~$1$ is either $\varepsilon$ or a propositional variable.

Rest of the constraints track the consistency of the sample with the subformulas and subexpressions of the guessed formula. The second table, consisting of Formulas~\ref{eq:epsilon} to ~\ref{eq:star}, provides constraints for consistency for regular expressions. The third table provides the constraints for propositional variables and rest of the PSL operators. The constraints for the LTL operators appearing in the third table have been taken from~\cite{DBLP:conf/fmcad/NeiderG18}.

\begin{figure*}[h]
\centering
\begin{tcolorbox}[width=0.7\textwidth, title=Structural Constraints, colback=white, sharp corners]
\scriptsize 
\begin{align}
\Big[ \bigwedge_{1\leq k \leq n} \bigvee_{\lambda \in \Lambda_P} x_{k, \lambda} \Big] 
&\land \Big[ \bigwedge_{1 \leq k \leq n} \bigwedge_{\lambda \neq \lambda' \in \Lambda_P} \lnot x_{k, \lambda} \lor \lnot x_{k, \lambda'} \Big]\label{eq:label_uniqueness}\\    
[\bigwedge\limits_{2\leq k\leq  n} \bigvee\limits_{1\leq \ell\leq k}l_{k,\ell}]
&\wedge[\bigwedge\limits_{2\leq k\leq  n}\bigwedge\limits_{1\leq \ell\leq \ell^\prime\leq n}\neg l_{k,\ell}\vee \neg l_{k,\ell^\prime}]\label{eq:left_child_uniqueness}\\
[\bigwedge\limits_{2\leq k\leq  n} \bigvee\limits_{1\leq \ell\leq k}r_{k,\ell}]
&\wedge[\bigwedge\limits_{2\leq k\leq  n}\bigwedge\limits_{1\leq \ell\leq \ell^\prime\leq n}\neg r_{k,\ell}\vee \neg r_{k,\ell^\prime}]\label{eq:right_child_uniqueness}\\
\bigwedge_{\substack{2 \leq k \leq n, 1 \leq \ell, \ell^\prime< k\\ \lambda\in\{+, \ast, \circ\}}} [ x_{k, \lambda} \land l_{k, \ell} \land r_{k,\ell^\prime} ] 
&\rightarrow \Big[ \bigvee_{\lambda^\prime \in \Lambda_R} x_{\ell, \lambda^\prime} \land \bigvee_{\lambda^\prime\in\Lambda_R} x_{\ell^{\prime}, \lambda'} \Big]\label{eq:regex_ordering}\\
\bigwedge_{\substack{2 \leq k \leq n, 1 \leq \ell, \ell^\prime< k\\ \lambda\in\{\ltlnext, \ltluntil, \neg, \lor \}}} [ x_{k, \lambda} \land l_{k, \ell} \land r_{k,\ell^\prime} ] 
&\rightarrow \Big[ \bigvee_{\lambda^\prime \in \Lambda_P} x_{\ell, \lambda^\prime} \land \bigvee_{\lambda^\prime\in\Lambda_P} x_{\ell^{\prime}, \lambda^\prime} \Big]\label{eq:LTL_ordering}\\
\bigwedge_{\substack{2 \leq k \leq n, 1 \leq \ell, \ell^\prime< k}} [ x_{k, \triggers} \land l_{k, \ell} \land r_{k,\ell^\prime} ] 
&\rightarrow \Big[ \bigvee_{\lambda^\prime \in \Lambda_R} x_{\ell, \lambda^\prime} \land \bigvee_{\lambda^\prime\in\Lambda_P} x_{\ell^{\prime}, \lambda^\prime} \Big]\label{eq:triggers_ordering}\\
x_{1,\epsilon}&\vee \bigvee\limits_{p\in\ap} x_{1, p}\label{eq:prop_ordering}
\end{align}
\end{tcolorbox}

\begin{tcolorbox}[width=0.7\textwidth, title=Constraints for Regular Expressions, colback=white, sharp corners]
\scriptsize
\begin{align}
\bigwedge\limits_{1\leq k\leq n}x_{k,\varepsilon}
&\rightarrow\Big[\bigwedge\limits_{0\leq i\leq j\leq \abs{uv^b} }z^{u,v}_{i,j,k}\leftrightarrow [i=j]\Big]\label{eq:epsilon}\\
\bigwedge\limits_{1\leq k\leq n}\bigwedge\limits_{p\in\ap}x_{k,p}
&\rightarrow\Big[\bigwedge\limits_{0\leq i\leq j\leq \abs{uv^b} }
\begin{cases}
z^{u,v}_{i,j,k}\text{ if }p\in uv^b[i,j)\\
\neg z^{u,v}_{i,j,k}\text{ if }p\not\in uv^b[i,j)
\end{cases}\Big]\label{eq:regex_prop}\\
\bigwedge\limits_{\substack{1\leq k\leq n \\ 1\leq \ell, \ell^\prime< k}}x_{k,+}\wedge l_{k,\ell}\wedge r_{k,\ell^\prime}
&\rightarrow\Big[\bigwedge\limits_{\substack{0\leq i \leq j \leq \abs{uv^b}}}\Big[z^{u,v}_{i,j,k}\leftrightarrow z^{u,v}_{i,j,\ell}\vee z^{u,v}_{i,j,\ell^\prime}\Big]\Big]\label{eq:union}\\
\bigwedge\limits_{\substack{1\leq k\leq n \\ 1\leq \ell, \ell^\prime< k}}x_{k,\circ}\wedge l_{k,\ell}\wedge r_{k,\ell^\prime}
&\rightarrow\Big[\bigwedge\limits_{\substack{0\leq i\leq j \leq \abs{uv^b} }}\Big[z^{u,v}_{i,j,k}\leftrightarrow \bigvee\limits_{i\leq t\leq j}z^{u,v}_{i,t,\ell}\wedge z^{u,v}_{t,j,\ell^\prime}\Big]\Big]\label{eq:concat} \\
\bigwedge\limits_{\substack{1\leq k\leq n \\ 1\leq \ell< k}}x_{k,\ast}\wedge l_{k,\ell}
&\rightarrow\Big[\bigwedge\limits_{\substack{0\leq i \leq j \leq \abs{uv^b}}}\Big[z^{u,v}_{i,j,k}\leftrightarrow [i=j]\vee \bigvee\limits_{i< t\leq j}z^{u,v}_{i,t,\ell}\wedge z^{u,v}_{t,j,k}\Big]\Big] \label{eq:star}
\end{align}
\end{tcolorbox}

\begin{tcolorbox}[width=0.7\textwidth, title=Constraints for LTL and triggers operators, colback=white, sharp corners]
\scriptsize
\begin{align}
\bigwedge\limits_{1\leq k \leq n}\bigwedge\limits_{p\in \ap}x_{k,p}
&\rightarrow\Big[\bigwedge\limits_{0\leq i< \abs{uv}}
\begin{cases}
y^{u,v}_{i,k}\text{ if }p\in uv[i,i) \\
\neg y^{u,v}_{i,k}\text{ if }p\not\in uv[i,i)
\end{cases}\Big]\label{eq:ltl_prop}\\
\bigwedge\limits_{\substack{1\leq k \leq n \\ 1\leq \ell\leq k}}x_{k,\neg}\wedge l_{k,\ell}
&\rightarrow\Big[\bigwedge\limits_{\substack{0\leq i < \abs{uv}}}\Big[y^{u,v}_{i,p}\leftrightarrow \neg y^{u,v}_{i,q}\Big]\Big] \label{eq:not}\\
\bigwedge\limits_{\substack{1\leq k \leq n \\ 1\leq \ell, \ell^\prime< k}}x_{k,\vee}\wedge l_{k,\ell}\wedge r_{k,\ell^\prime}&\rightarrow\Big[\bigwedge\limits_{\substack{0\leq i < \abs{uv}}}\Big[y^{u,v}_{i,k}\leftrightarrow y^{u,v}_{i,\ell}\vee y^{u,v}_{i,j,\ell^\prime}\Big]\Big]\label{eq:or}\\
\bigwedge\limits_{\substack{1\leq k \leq n \\ 1\leq \ell< k}}x_{k,\ltlnext}\wedge l_{k,\ell} 
&\rightarrow\Big[\bigwedge\limits_{\substack{0\leq i< \abs{uv}-1}}\Big[y^{u,v}_{i,k}\leftrightarrow y^{u,v}_{i+1,\ell}\Big]\wedge \Big[y^{u,v}_{\abs{uv}-1,k}\leftrightarrow y^{u,v}_{\abs{u},\ell}\Big]\Big]\label{eq:next}\\
\bigwedge\limits_{\substack{1\leq k \leq n \\ 1\leq \ell, \ell^\prime< k}}x_{k,\ltluntil}\wedge l_{k,\ell}\wedge r_{k,\ell^\prime}
&\rightarrow\Big[\bigwedge\limits_{\substack{0\leq i< \abs{u}}}\Big[y^{u,v}_{i,k}\leftrightarrow \bigvee\limits_{i\leq j< \abs{uv}}\Big[y^{u,v}_{j,\ell^\prime}\wedge  \bigwedge\limits_{i\leq t< j} y^{u,v}_{t,\ell}\Big]\Big]\nonumber\\
&\wedge \Big[\bigwedge\limits_{\substack{\abs{u}\leq i< \abs{uv}}}\Big[y^{u,v}_{i,k}\leftrightarrow \bigvee\limits_{\abs{u}\leq j< \abs{uv}}\Big[y^{u,v}_{j,\ell^\prime}\wedge  \bigwedge\limits_{k\in\mathcal{I}^{u,v}(i,j)} y^{u,v}_{t,\ell}\Big]\Big]\label{eq:until}\\
\text{\textit{where, }}&\mathcal{I}_{u,v}(i,j)=
\begin{cases}
\{i,\cdots, j-1\} &\text{ for } i\leq j \\
\{\abs{u},\cdots,j-1\}\cup\{i,\cdots, \abs{uv-1}\} &\text{ for } i>j
\end{cases}\nonumber\\
\bigwedge\limits_{\substack{1\leq k \leq n \\ 1\leq \ell, \ell^\prime< k}}x_{k,\triggers}\wedge l_{k,\ell}\wedge r_{k,\ell^\prime} 
&\rightarrow \Big[\bigwedge\limits_{0\leq i< \abs{uv}}\Big[y^{u,v}_{i,k}\leftrightarrow \bigwedge\limits_{i\leq j\leq \abs{uv^{b}}}\Big[z^{u,v}_{i,j,\ell}\rightarrow y^{u,v}_{M_{u,v}(j-1),\ell^\prime} \Big]\Big]\Big]\label{eq:triggers}
\end{align}
\end{tcolorbox}
\caption{List of all constraints used in the SAT encoding}
\label{fig:all-constraints}
\end{figure*}
\FloatBarrier

\section{Correctness of the Learning Algorithm}
\label{app:correctness-learning-algorithm}

In order to prove Theorem~\ref{thm:correctness-learning-algorithm}, we show the following lemma, which establishes that the formula $\Phi_n^\mathcal{S}$ indeed has the desired properties.

\begin{lemma} \label{lem:correctness-propositional-formula}
Let $\mathcal{S} = (P, N)$ be a sample, $n \in \mathbb N \setminus \{ 0 \}$, and $\Phi_n^\mathcal S$ be the propositional formula used in Algorithm~\ref{alg:sat-learner}.
Then, the following holds:
\begin{enumerate}
	\item If there exists a PSL formula $\varphi^{\mathcal S}$ of size $n$, that is consistent with $\mathcal S$, then the propositional formula $\Phi_n^\mathcal S$ is satisfiable.
	\item If $V \models \Phi_n^\mathcal S$, then $\varphi^{V}$ is a PSL formula of size $n$ that is consistent with $\mathcal S$.
\end{enumerate}
\end{lemma}
\begin{proof}

For proving the first statement, we use the syntax DAG of the formula $\varphi^{\mathcal S}$ (indexed using $1\ldots n$), to formulate a valuation $V$ for the propositional variables in $\Phi^{\mathcal S}_n$. We use $\varphi_k^{\mathcal S}$ to refer to the subformula rooted at the Node~$k$. Alternatively, we use $\rho_k^{\mathcal S}$ for the subexpression at Node~$k$. 
    \begin{itemize}
    \item We set $V(x_{k,\lambda})=1$ if and only if the Node~$k$ node is labeled by $\lambda$.   
    \item We set $V(l_{k,\ell})=1$ if and only if Node~$\ell$ is the left child of the Node~$k$ and also, set $V(r_{k,\ell})=1$ similarly for the right child.
    \item We set $V(y^{u,v}_{i,k})=1$ if and only if $uv^\omega[i,\infty)\modelspsl \varphi_k^{\mathcal S}$, when label at Node~$k$ is in $\Lambda_P$, but is not an operator from regular expressions.
    \item We set $V(z^{u,v}_{i,j,k})=1$ if and only if $uv^b[i,j)\matches \rho_k^{\mathcal S}$, when label at Node~$k$ is in $\Lambda_R$. 
    \end{itemize}
    
Firstly, it can be seen that $V\models \Phi_n^{\text{str}}$, since the formulated valuation ensures the uniqueness of the labels of nodes, as well as that of their left and right children. 
The ordering constraints are also satisfied, since $\varphi^{\mathcal S}$ is a valid PSL formula. 
Further, $V\models \Phi^{u,v}_n$ for all $uv^\omega\in P\cup N$, since, the values of the variables
$z^{u,v}_{i,j,k}$ correspond to the matching of subexpressions $\rho_k^{\mathcal S}$ with $uv^b[i,j)$ for regular operators and atomic expressions; while the values of the variables $y^{u,v}_{i,k}$ correspond to the satisfaction of subformulas $\varphi_k^{\mathcal S}$ on $uv^\omega[i,\infty)$ for propositions and rest of the PSL operators. 
Finally, the fact that $\varphi^{\mathcal S}$ is consistent with $\mathcal S$ implies $V\models y^{u,v}_{0,n}$ for each word in $P$ and $V\models\neg y^{u,v}_{0,n}$ for each word in $N$. This proves $V\models \Phi^{\mathcal S}$.

In the second statement, observe that, since, $V\models \Phi^S_n$, we have $V\models \Phi^{\text{str}}_n$ as well. 
Hence, the valuation of the variables $x_{k,\lambda}$, $l_{k,\ell}$, and $r_{k,\ell}$ encode a syntax DAG from which we obtain the PSL formula $\varphi^V$. 
Additionally, the ordering constraints ensure proper ordering of the operators.
Next, we define $\varphi_k^V$, to be the subformula of $\varphi^V$ rooted at the Node~$k$, if the node is labelled by propositions, LTL or triggers operator while $\rho_k^V$ to be the subexpression at Node~$k$ for regular operators or atomic expressions. Now, it needs to be shown that $\varphi^V$ is indeed consistent with the sample $\mathcal S$. 
To this end, we show ${V(y^{u,v}_{i,k})=1}$ if and only if ${uv^\omega[i,\infty)\matches \varphi^V_k}$ for any $i\in \{0,\cdots, \abs{uv}-1\}$ for the subformulas; and ${V(z^{u,v}_{i,j,k})=1}$ if and only if ${uv^b[i,j)\matches \rho^V_k}$ for any $i, j\in \{0,\cdots, \abs{uv^b}\}$ for the subexpressions. This proof proceeds via induction on the structure of $\varphi^V$. 

For proving the base cases (that is for $\varepsilon$ and propositional variables), we use the constraints for '$\varepsilon$' and propositions (presented in the Figure~\ref{fig:all-constraints}) followed by the semantics of PSL formulas. The induction on the operators in PSL proceeds similarly, except that here, we use the inductive hypothesis for subformulas (or subexpressions) of smaller size, to relate the information derived from the appropriate constraints, to the semantics of PSL. The proofs for the different cases is presented below.

\begin{itemize}
	\item In the \textit{base} case $\rho^V_k=\varepsilon$, we have $V(x_{k,\epsilon})$ set to 1, and thus, we make the following deductions:
    \begin{align*}
    V(z^{u,v}_{i,j,k})=1&\iff i=j\\
                    &\iff uv^b[i,j)\matches\varepsilon
    \end{align*}
    \item In the \textit{base} case $\varphi^V_k=p$ or $\rho^V_k=p$, we have $V(x_{k,p})$ set to 1, and thus, we can make the following deductions:
    \begin{align*}
    V(y^{u,v}_{i,k})=1&\iff p\in uv^\omega[i]\\
                    &\iff uv^\omega[i,j)\modelspsl p\\        
    V(z^{u,v}_{i,j,k})=1&\iff p\in uv^\omega[i,i)\\
                    &\iff uv^b[i,j)\matches p
    \end{align*}
    \item In the case $\rho^V_k=\rho^V_\ell+\rho^V_{\ell^\prime}$, we have $V(x_{k,+})$, and $V(l_{k,\ell})$, $V(r_{k,\ell^\prime})$ set to 1, and thus, we make the following deductions:
    \begin{align*}
    V(z^{u,v}_{i,j,k})=1 &\iff V(z^{u,v}_{i,j,\ell})=1 \text{ or } V(z^{u,v}_{i,j,\ell^\prime})=1\\
    &\iff uv^b[i,j)\matches \rho^V_\ell\text{ or } uv^b[i,j)\matches \rho^V_{\ell^\prime}\\
    &\iff uv^b[i,j)\matches \rho^V_\ell+\rho^V_{\ell^\prime}
    \end{align*}
    \item In the case $\rho^V_k=\rho^V_\ell\circ \rho^V_{\ell^\prime}$, we have $V(x_{k,\circ})$, $V(l_{k,\ell})$, and $V(r_{k,\ell^\prime})$ all set to 1, and thus, we make the following deductions:
    \begin{multline*}
    V(z^{u,v}_{i,j,k})=1 \\
    \begin{aligned}
    &\iff \exists t\in \mathbb{N},\ 1\leq t\leq \abs{uv^b}, \\
    &\hspace{2cm} V(y^{u,v}_{i,t,\ell})=1 \text{, and } V(y^{u,v}_{t,j,q^{\prime}})=1\\
    &\iff \exists t\in \mathbb{N},\ 1\leq t\leq \abs{uv^b}, \\
    &\hspace{2cm} uv^b[i,t)\matches \rho^V_\ell \text{, and } uv^b[k,j)\matches \rho^V_{\ell^\prime}\\
    &\iff uv^b[i,j)\matches \rho^V_\ell\circ \rho^V_{\ell^\prime}
    \end{aligned}    
    \end{multline*}   
    \item In the case $\rho^V_k=(\rho^V_\ell)^{*}$, we have $V(x_{k,*})$ and $V(l_{k,\ell})$ set to 1, and thus, we make the following deductions:
    \begin{multline*}
    V(z^{u,v}_{i,j,k})=1\\
    \begin{aligned}
    &\iff 
    \begin{cases} 
    i=j; \text{ or}\\ 
    \exists t\in \mathbb N,\ 1\leq t\leq \abs{uv^b},\\
    \hspace{0.5cm} V(y^w_{i,k,q})=1 \text{, and }V(y^w_{k,j,p})=1
    \end{cases}\\
    &\iff     
    \begin{cases} 
    i=j; \text{ or}\\ 
    \exists t\in \mathbb N,\ 1\leq t\leq \abs{uv^b},\\
    \hspace{0.5cm} uv^b[i,t)\matches \rho^V_\ell \text{, and }uv^b[t,j)\matches (\rho^V_k)^*
    \end{cases}\\
    &\iff uv^b[i,j)\matches (\rho^V_\ell)^*
    \end{aligned} 
    \end{multline*}
	\item In the case $\varphi^V_k=\neg\varphi^V_\ell$, we have $V(x_{k,\neg})$ and $V(l_{k,\ell})$ set to 1, and thus, we make the following deductions:
    \begin{align*}
    V(y^{u,v}_{i,k})=1 &\iff V(y^{u,v}_{i,\ell})=0 \\
    &\iff uv^\omega[i,\infty)\nvDash \varphi^V_\ell\\
    &\iff uv^\omega[i,\infty)\modelspsl \neg\varphi^V_\ell
    \end{align*}
	\item In the case $\varphi^V_k=\varphi^V_\ell \lor \varphi^V_{\ell^\prime}$, we have $V(x_{k,\lor})$, $V(l_{k,\ell})$, and $V(r_{k,\ell^\prime})$ all set to 1, and thus, we make the following deduction:
    \begin{multline*}
    V(y^{u,v}_{i,k})=1\\
    \begin{aligned}
    &\iff V(y^{u,v}_{i,\ell})=1 \text{ or } V(y^{u,v}_{i,\ell^{\prime}})=1 \\
    &\iff uv^\omega[i,\infty)\modelspsl \varphi^V_\ell\text{ or } uv^\omega[i,\infty)\modelspsl \varphi^V_{\ell^\prime}\\
    &\iff uv^\omega[i,\infty)\modelspsl \varphi^V_\ell\lor \varphi^V_{\ell^\prime}
    \end{aligned}
    \end{multline*}
    \item In the case $\varphi^V_k=\ltlnext \varphi^V_\ell$, we have $V(x_{k,\ltlnext})$ and $V(l_{k,\ell})$ set to 1 and thus, we make the following deductions:
    \begin{multline*}
    V(y^{u,v}_{i,k})=1 \\
    \begin{aligned}    
    &\iff 
    \begin{cases}
    V(y^{u,v}_{i+1,\ell})=1\ \text{for}\ 0\leq i < \abs{uv}-1 \\
    V(y^{u,v}_{\abs{u}, \ell})=1\ \text{for}\ i= \abs{uv} -1    
    \end{cases}\\
    &\iff
    \begin{cases}
    uv^\omega[i+1, \infty)\modelspsl \varphi^V_\ell\ \text{for}\ 0\leq i < \abs{uv}-1\\
    uv^\omega[\abs{u}, \infty)\modelspsl \varphi^V_\ell\ \text{for}\ i=\abs{uv}-1
    \end{cases}\\
    &\iff uv^\omega[i,\infty)\modelspsl \ltlnext\varphi^V_\ell
    \end{aligned}
    \end{multline*}
    \item In the case $\varphi^V_k=\varphi^V_\ell \ltluntil \varphi^V_{\ell^\prime}$, we have $V(x_{k,\ltluntil})$, $V(l_{k,\ell})$, and $V(r_{k,\ell^\prime})$ all set to 1, and thus, we make the following deductions:
    \begin{multline*}
    V(y^{u,v}_{i,k})=1\\
    \begin{aligned}
    &\iff
    \begin{cases}
    \exists j,\ i\leq j\leq \abs{uv}-1,\ V(y^{u,v}_{j,\ell^\prime})=1, \text{and}\\
    \forall t, i\leq t < j,\ V(y^{u,v}_{t,\ell})=1\
    \text{for}\ i< \abs{u} \\[0.2cm]
     \exists j,\ \abs{u}\leq j< \abs{uv},\ V(y^{u,v}_{j,\ell^\prime})=1,\ \text{and}\\
     \forall t\in\mathcal{I}^{u,v}(i,j),\ V(y^{u,v}_{t,\ell})=1\ \text{for}\ i\geq \abs{u}\
    \end{cases}\\
    &\iff
    \begin{cases}
    \exists j,\ i\leq j\leq \abs{uv}-1,\ uv^\omega[j,\infty)\modelspsl \varphi^{V}_{\ell^\prime},\ \text{and}\\ \forall t, i\leq t < j,\ uv^\omega[t,\infty)\modelspsl \varphi^{V}_\ell\ \text{for}\ i< \abs{u}\\[0.2cm]
     \exists j,\ \abs{u}\leq j< \abs{uv},\ uv^\omega[j,\infty)\modelspsl \varphi^{v}_{q^{\prime}},\ \text{and}\\
     \forall t\in\mathcal{I}_{u,v}(i,j),\ uv^\omega[t,\infty)\modelspsl \varphi^{V}_\ell\ \text{for}\ i\geq \abs{u}\
    \end{cases}\\
    &\iff uv^\omega[i,j)\modelspsl \varphi^V_\ell\ltluntil \varphi^V_{\ell^\prime}
    \end{aligned}
    \end{multline*}
	\item In the case $\varphi^V_k=\rho^V_\ell \triggers \varphi^V_{\ell^{\prime}}$, we have $V(x_{k,\triggers})$, $V(l_{k,\ell})$, and $V(r_{k,\ell^\prime})$ set to 1, and thus, we make the following deductions:
	\begin{multline*}
	V(y^{u,v}_{i,k})=1\\
	\begin{aligned}
	&\iff \forall j\in{\mathbb{N}},\ i\leq j\leq \abs{uv^b},\\
	&\hspace{0.5cm}\text{if }V(z^{u,v}_{i,j,\ell})=1\text{ then } V(y^{u,v}_{M_{u,v}(j-1),\ell^{\prime}})=1\\
	&\iff \forall j\in{\mathbb{N}},\ i\leq j\leq \abs{u}+b\abs{v},\\
	&\hspace{0.5cm}\text{if }uv^b[i,j)\matches \rho^V_\ell\text{ then } uv^\omega[M_{u,v}(j-1),\ell^\prime)\modelspsl \varphi_\ell^{\prime}\\
	&\iff \forall j\in{\mathbb{N}},\ i\leq j\leq \abs{u}+b\abs{v},\\
	&\hspace{0.5cm}\text{if }uv^b[i,j)\matches \rho^V_\ell\text{ then } uv^\omega[j-1,\ell^\prime)\modelspsl \varphi_\ell^{\prime}\\
	&\iff uv^\omega[i,\infty)\modelspsl \rho^V_\ell\triggers\varphi^V_{\ell^{\prime}}
	\end{aligned}
	\end{multline*}   
\end{itemize}

\end{proof}

\end{document}